\documentclass[11pt]{article}
\usepackage{amsmath, amsfonts, amssymb, amsthm, url}

\theoremstyle{lemma}
\newtheorem{lemma}{Lemma}[section]

\newcommand*\colvec[3]{
    \begin{pmatrix}#1\\#2\\#3\end{pmatrix}
}

\usepackage{textcomp}
\newcommand{\textapprox}{\raisebox{0.5ex}{\texttildelow}}

\begin{document}

\title{Secure multi-party linear regression \\at plaintext speed}
\author{Jonathan M. Bloom}
\date{Broad Institute of MIT and Harvard\\\url{jbloom@broadinstitute.org}}
\maketitle

\section{Introduction}
We describe a simple and efficient distributed algorithm for multi-party linear regression. First, each party linearly {\em compresses} the sample dimension of their data from the number of samples to the number of covariates. Second, the parties {\em combine} these compressed representations to compute exact statistics. Using secure multi-party computation (SMC) for the second step, the algorithm is provably privacy-preserving and as efficient as working in plaintext, asymptotically as the number of samples grows. We then extend this algorithm to the context of efficiently testing many features for association with a response while adjusting for a fixed set of covariates.

Our motivating use case involves many institutions, hospitals, or biobanks, each with the genomes and traits for a set of individuals, unable to share or colocate their data for reasons of privacy or cost or timing\footnote{By timing, we mean when a new center or batch of samples comes online after the initial analysis. In this case, our regression scheme updates the existing statistics at incremental cost; that is, at a cost that is independent of the original number of samples and proportional to number of samples in the new batch.}. Nonetheless, these centers would like to efficiently run a genome-wide association study (GWAS), testing each of millions of variants for association with each trait, using the combined data in order to maximize statistical power. In this context, one typically computes and includes principal components scores as covariates to control for confounding by ancestry. These scores may be computed securely by each center using a public linear projection defined by PCA on a reference panel. In sum, our algorithm enables optimally efficient and scalable SMC GWAS in the multi-center setting\footnote{It is often useful to compute PCA on the combined samples themselves rather than from a reference panel. Doing so securely at scale was intractable when this note was first written in May 2017. While there has since been very exciting progress on the efficiency of secure PCA and regression under the contrasting setup in which each individual secret-shares their own genome, these methods remain many orders of magnitude slower than plaintext computation: Hyunghoon Cho, David J Wu, Bonnie Berger. \textit{Secure genome-wide association analysis using multiparty computation}. Nature Biotechnology. Vol. 36, p. 547-551. May 2018.}.

One can imagine a future in which distributed multi-party GWAS, with or without SMC, is simple to perform efficiently on compressed representations in the cloud, with only incremental cost as new centers or batches of samples come online. Promising signals might then incentivize open or cryptographic collaboration on that sliver of the data in order to bring to bear more sophisticated quality control and statistical models en route to a joint search for biological mechanism and therapeutic target. We aspire to help realize this future with Hail (\url{hail.is}), an open-source project used by academia and industry for scalable genomic analysis. Given the centrality of linear regression and feature selection in data science, we also hope these natural algorithms will find broad utility in industry. \\

\noindent {\bf Organization.} Section \ref{sec:linreg} extends single-party linear regression to the multi-party setting. Section \ref{sec:scan} describes the projection trick (in the single-party setting) for efficiently scanning through many features for association with a response while adjusting for a fixed set of covariates, as currently implemented in Hail. Section \ref{sec:multiscan} uses the TSQR algorithm to extend the projection trick to the multi-party setting.\\

\noindent {\bf Code.} For a demo of the algorithm, nicknamed the {\em distributed association scan hammer} (DASH), visit \url{github.com/jbloom22/DASH}. \\

\noindent {\bf Acknowlegements.} This note was inspired by discussions with Alex Bloemendal on linear regression, Ben Neale on statistical genetics, and Hail lead Cotton Seed on distributed systems. The author is a member of the Hail Team in the Neale Lab and the Stanley Center for Psychiatric Research at the Broad Institute of MIT and Harvard. Visit Models, Inference and Algorithms (\url{broadinstitute.org/mia}) for more bio-inspired computation.

\section{Linear regression}
\label{sec:linreg}

Consider linear regression\footnote{For an excellent review of linear regression, see Chapter 3 of {\em Elements of Statistical Learning}, available for free at \url{https://web.stanford.edu/~hastie/ElemStatLearn/}.
} with $N$ samples, $K$ covariates, and data:
\begin{itemize}
\item $y$, an $N$-dimensional response vector.
\item $C$, an $N \times K$ matrix of linearly independent covariate vectors.
\end{itemize}
We do not regard the intercept as special; rather it may be included as a covariate vector of ones. For $1 \leq i \leq N$, the model posits
$$y_i = \gamma_1 c_{i1} + \cdots + \gamma_K c_{iK} + \varepsilon_i$$
where the $\gamma_i$ are coefficient parameters to be inferred and $\varepsilon_i \sim \mathcal{N}(0, \tau^2)$. Equivalently, we may regard $y$ as a single draw from an $N$-dimensional, spherical normal distribution with mean $C\gamma$ and variance parameter $\tau^2$:
\begin{align}
y \sim \mathcal{N}(C \gamma, \tau^2 I)
\end{align}
The distribution of $\gamma$ given $y$, $C$ and $\tau^2$ is normal with mean and variance
$$\hat\gamma = (C^\intercal C)^{-1} C^\intercal y, \quad \rm{Var}(\gamma) = (C^\intercal C)^{-1}\tau^2.$$
If unknown, $\tau^2$ is approximated using the unbiased estimator
$$\hat\tau^2 = \frac{|y - C \hat\gamma|^2}{N - K} = \frac{y \cdot y - \hat\gamma^\intercal (C^\intercal C) \hat\gamma}{N - K},$$
where the last equality follows from the Pythagorean theorem. The usual statistics of linear regression follow. E.g., the standard error of $\gamma_k$ is given by $\hat\tau \sqrt{v_k}$ where $v_k$ is the $k$th diagonal element of $(C^\intercal C)^{-1}$.

These formulae realize linear regression as routed through two stages. \\

\noindent {\bf Compress:} Compute all pairwise dot products of all $N$-vectors:
$$y^\intercal y, \qquad C^\intercal y, \qquad C^\intercal C.$$

\noindent {\bf Combine:} Compute all statistics from $N$ and these quantities. \\

\noindent  Compression has computational complexity $\mathcal{O}(NK^2)$ whereas combining has complexity $\mathcal{O}(K^3)$. In particular, {\em combining is independent of sample size}; for applications in which the number of samples far exceeds the number of covariates, the compression stage dominates.

If the samples are divided across $P$ parties (or data partitions), each with the response and covariate data for $N_p$ samples, then we simply modify the stages as follows. \\

\noindent {\bf Compress within:} Compute all pairwise dot products of all $N_p$-vectors:
$$y_p^\intercal y_p, \qquad C_p^\intercal y_p, \qquad C_p^\intercal C_p.$$

\noindent {\bf Combine across:} Sum across parties to obtain $$N, \qquad y^\intercal y, \qquad C^\intercal y, \qquad C^\intercal C.$$ Compute all statistics from these quantities. \\

\noindent Compression within each party has complexity $\mathcal{O}(N_i K^2)$ so the total complexity is the same as single-party. Combining across parties has computational and network complexity $\mathcal{O}(PK^2)$ for summing and then computational complexity $\mathcal{O}(K^3)$ for computing the statistics as in the single-party case. Hence, compressing is locally parallelized by party while combining is networked and independent of sample size. To theoretically guarantee that the parties learn no information about one another's data besides the final statistics, one need only apply cryptography (e.g., SMC) in the combine stage. That is, {\bf compress in plaintext, combine with crypto}.

\section{Single-party association scan}
\label{sec:scan}

One often wants to efficiently test many features for association with a response while adjusting for a fixed set of covariates. Let's formalize this problem under the name \textit{association scan}. Consider the data:
\begin{itemize}
\item $y$, an $N$-dimensional \textit{response} vector.
\item $X$, an $N \times M$ matrix of $M$ \textit{transient} covariate vectors.
\item $C$, an $N \times K$ matrix of $K$ linearly independent \textit{permanent} covariate vectors.
\end{itemize}
Let $X_m$ denote column $m$ of $X$, e.g., transient covariate vector $m$. Independently for $1 \leq m \leq M$, we regard $y$ as a single draw from an $N$-dimensional normal distribution with mean parameters $\beta_m$ and $\gamma_m$ and variance parameter $\tau_m^2$:
\begin{align*}
y \sim \mathcal{N}(X_m \beta_m  + C \gamma_m, \tau_m^2 \mathit{I})
\end{align*}
Let $\hat{\beta}_m$ be the maximum likelihood estimate for the transient coefficient $\beta_m$ and let $\hat{\sigma}_m$ be the standard error of this estimate. \\

\noindent \textbf{Association scan problem:} Determine the vectors $\hat\beta = (\hat{\beta}_1, \dots, \hat{\beta}_M)$ and $\hat\sigma = (\hat{\sigma}_1, \dots, \hat{\sigma}_M)$ efficiently and scalably. \\

\noindent Note that under the null hypothesis $\beta_m = 0$, the statistic $\hat\beta_m / \hat\sigma_m$ is drawn from a $t$-distribution with $N - K - 1$ degrees of freedom, so the t-statistics and p-values immediately follow. \\

\noindent \textbf{Example:} In genome wide association studies, which scan the genome for correlation of genetic and phenotypic variation, we have $N$ samples (individuals), $M$ common variants to test one by one, and $C$ sample-level covariates like intercept, age, sex, batch, and principal component scores. Typically $N$ is $10^2$ to $10^6$, $M$ is $10^5$ to $10^8$, and $K$ is 1 to 30. In human gene burden tests, $M$ is the number of genes (\textapprox{20K}).

Often one wants to test many traits for association with each variant. All algorithms herein generalize efficiently on vectorized hardware by promoting the vector $y$ to a matrix $Y$ in addition to treating $X$ as a matrix. For example, we implemented the algorithm in this section in Hail to compute the largest genetic association to date in hours (\url{http://www.nealelab.is/uk-biobank/}), testing \textapprox{4K} traits at each of \textapprox{13M} variants across \textapprox{360K} individuals using \textapprox{25} permanent covariates. Hail is also used for studies of the impact of genetic variation on gene expression, corresponding to \textapprox{20K} traits. \\

Let $Q$ be an $N \times K$ matrix whose columns form an orthonormal basis for the column space of $C$. Let $X \cdot y$ denote the vector with values $X_m \cdot y$. Let $X \cdot X$ denote the vector with values $X_m \cdot X_m$. Let $\hat\beta^2$ denotes coordinate-wise squaring of $\hat\beta$. The following equations, derived geometrically in the Appendix, yield an efficient distributed algorithm for a single-party association scan.

\begin{lemma}
\label{closedform}
The association scan problem is solved by:
\begin{align*}
\hat\beta &= \frac{X \cdot y - Q^\intercal X \cdot Q^\intercal y}{X \cdot X - Q^\intercal X \cdot Q^\intercal X} \\
\hat\sigma^2 &= \frac{1}{N - K - 1}\left(\frac{y \cdot y - Q^\intercal y \cdot Q^\intercal y}{X \cdot X - Q^\intercal X \cdot Q^\intercal X} - \hat\beta^2\right)
\end{align*}
\end{lemma}

\noindent \textbf{Single-party distributed algorithm:} We assume the columns of $X$ are distributed across machines with $C$ total cores.
\begin{enumerate}
\item Compute the QR decomposition of $C$ and broadcast $Q$.
\item Compute and broadcast $y \cdot y$, $Q^\intercal y$, and $Q^\intercal y \cdot Q^\intercal y$.
\item In parallel, compute $X \cdot X$, $Q^\intercal X$, $Q^\intercal X \cdot Q^\intercal y$ and $Q^\intercal X \cdot Q^\intercal X$.
\item In parallel, compute $\hat\beta$ and $\hat\sigma$ as in Lemma \ref{closedform}.
\end{enumerate}

\noindent Computing $Q$ and $Q^\intercal X$ dominate the computational complexity as
\label{bigO}
\begin{align}
\mathcal{O}\left(NK^2 + \frac{NKM}{C}\right).
\end{align}
In practice we consider $K$ as a small constant so the complexity is
\label{bigO2}
\begin{align}
\mathcal{O}\left(\frac{NM}{C}\right),
\end{align}
i.e. that of reading the data and therefore best possible with no further assumptions on the entropy of $X$. For further gains, the columns of $X$, if sparse, can be packed sparsely to reduce the dominating complexity of computing $Q^\intercal X$.

\section{Multi-party association scan}
\label{sec:multiscan}

Suppose the $N$ samples are divided among $P$ parties who are not willing or able to share their data. In this case, analysts typically resort to meta-analyzing within-party estimates, with loss of power due to noisy standard errors as well as between-group heterogeneity (c.f. Simpson's paradox). Being power hungry, we instead solve the: \\

\noindent \textbf{Secure multi-party association scan problem:} Securely determine the vectors $\hat\beta = (\hat{\beta}_1, \dots, \hat{\beta}_M)$ and $\hat\sigma = (\hat{\sigma}_1, \dots, \hat{\sigma}_M)$ efficiently and scalably while communicating only $\mathcal{O}(M)$ bits inter-party.\\

Here we again we consider $K$ a small constant. Note that $\mathcal{O}(M)$ is best possible since all parties must receive the results. In fact, our multi-party algorithm has the same distributed computational complexity as the plaintext single-party algorithm in the last section. With provably security, this remains true asymptotically in sample size. \\

To extend the projection trick to the multi-party setting, we use the geometric idea underlying the TSQR algorithm\footnote{See for example \textit{Tall and skinny QR factorizations in MapReduce architectures}, \url{https://pdfs.semanticscholar.org/747c/a08cbf258da8d2b89ba31f24bdb17d7132bb.pdf}}. For notational simplicity, we derive this idea for $P = 3$. Alice, Bob, and Carla have $N_a$, $N_b$, and $N_c$ samples, respectively, so the data has the form:
$$y = \colvec{y_a}{y_b}{y_c}, \quad X = \colvec{X_a}{X_b}{X_c}, \quad C = \colvec{C_a}{C_b}{C_c}$$
Our algorithm requires that $C_a$, $C_b$, and $C_c$ have full column-rank, so that QR decompositions are unique (requiring that $R$ have positive diagonal entries). Let $R_a$, $R_b$, and $R_c$ be the $R$ matrices in the QR decompositions of $C_a$, $C_b$, and $C_c$, respectively.
\begin{lemma} The QR decompositions of $C$ and
$$\colvec{R_a}{R_b}{R_c}$$
have the same $R$ matrix.
\end{lemma}
\begin{proof}
Let $Q^\prime R$ be the QR decomposition of the latter matrix. Then:
$$C = \colvec{C_a}{C_b}{C_c} = \colvec{Q_a R_a}{Q_b R_b}{Q_c R_c} = \begin{pmatrix} 
Q_a & 0 & 0 \\
0 & Q_b & 0 \\
0 & 0 & Q_c
\end{pmatrix} \colvec{R_a}{R_b}{R_c} = \begin{pmatrix} 
Q_a & 0 & 0 \\
0 & Q_b & 0 \\
0 & 0 & Q_c
\end{pmatrix} Q^\prime R.$$
The composition of isometries is an isometry, so combining the first and second matrices in the final expression gives the QR decomposition of $C$.
\end{proof}

\noindent \textbf{Multi-party distributed algorithm:} In the first step, each party linearly compresses the sample dimension of their data from the number of samples to the number of covariates. In the second step, the parties combine these compressed representations to compute exact statistics. \\

\noindent {\bf Compress within:} Compute the following dot products of $N_p$-vectors:
$$y_p^\intercal y_p,  \qquad X_p^\intercal y_p, \qquad X_p \cdot X_p, \qquad C_p^\intercal y_p, \qquad C_p^\intercal X_p.$$
Compute $R_p$ in the QR decomposition of $C_p$.  \\

\noindent {\bf Combine across:} Sum across parties to obtain
$$N, \qquad y^\intercal y, \qquad X^\intercal y, \qquad X \cdot X, \qquad C^\intercal y, \qquad C^\intercal X.$$
Compute $R$ for the vertical stack of the $R_p$ and use it to compute $$Q^\intercal y = (R^{-1})^\intercal(C^\intercal y), \qquad Q^\intercal X = (R^{-1})^\intercal(C^\intercal X).$$
Finally, compute all statistics from these quantities as in the single-party association scan. \\

Again, compressing is locally parallelized by party while combining is networked and independent of sample size. To theoretically guarantee that the parties learn no information about one another's data besides the final statistics, one need only apply cryptography (e.g., secure multiparty computation) in the combine stage which parallelizes over transient covariates.

Note that responses and permanent covariates play nearly identical roles in the compression stage. In fact, having run compression for a set of responses and permanent covariates, one can choose which to use in the model {\em without having to re-run compression}. Rather, each party need only compute the corresponding $R_p$ matrices.

Adding an intercept covariate is equivalent to mean-centering $y$ and each column of $C$. Adding an intercept for each party (i.e., $P$ indicator covariates to control for batch effects) is equivalent to mean centering $y$ and each column of $C_a$, $C_b$, and $C_c$ independently.

These algorithms efficiently generalize to the case of multiple transient covariants, such as interaction terms, or multiple responses, as arise in biobank and gene expression (eQTL) studies. Gene burden tests, in which gene scores are computed as linear combinations of genotypes, integrate very efficiently with this approach, since they involve linear projection of genomes on the variant axis rather than the sample axis, and matrix multiplication is associative. If kinship may be shared, then these algorithms extend efficiently to linear mixed models as well; Hail's implementation uses the eigendecomposition of kinship to reduce the problem to a single-party association scan. If parameter updates may be shared, then these algorithms suggest efficient primitives for iterative optimization as arises in generalized linear models and deep learning.

\section{Appendix}
Here we give a short geometric proof of Lemma \ref{closedform}, a version of the classic result that the statistics for each coefficient in linear regression may be computed by iterative regression.
\begin{proof}
To simplify notation, let $x = X_m$, $\beta = \hat\beta_m$, $\gamma = \hat\gamma_m$, and $\varepsilon = \hat\varepsilon_m$. Geometrically, linear regression is orthogonal projection in $\mathbb{R}^N$ of the response vector $y$ onto $\rm{span}(x, C)$. Let $P$ be the orthogonal projection from $\mathbb{R}^N$ to $C^\bot$. Since $\varepsilon \in \rm{span}(x, C)^\bot$ and $Px = x - (x - Px) \in \rm{span}(x, C)$, we have $PC = 0$, $P\varepsilon = \varepsilon$, and $Px \cdot \varepsilon = 0$. So applying $P$ to both sides of
$$y = x \beta + C \gamma + \varepsilon$$
gives an orthogonal decomposition
$$Py = (Px)\beta + \varepsilon.$$
Hence regressing $Py$ on $Px$ yields the same $\beta$ and $\varepsilon$ as regressing $y$ on both $x$ and $C$, i.e.
$$\beta = \frac{Px \cdot Py}{Px \cdot Px}, \quad \sigma^2 = \frac{\varepsilon \cdot \varepsilon}{N - K - 1} .$$
The formulae now follow from the Pythagorean theorem and the explicit form $P = I - QQ^\intercal$ with $Q$ an isometry.
\end{proof}

\end{document}